\documentclass[letterpaper]{article}

\usepackage{floatrow}
\newfloatcommand{capbtabbox}{table}[][\FBwidth] %

\usepackage{algorithm}
\usepackage{algorithmic}

\usepackage{aaai19}
\usepackage{times}
\usepackage{helvet}
\usepackage{courier}
\usepackage[utf8]{inputenc} 
\usepackage[T1]{fontenc}    
\usepackage{hyperref}       
\usepackage{url}            
\usepackage{booktabs}       
\usepackage{amsfonts}       
\usepackage{nicefrac}       
\usepackage{microtype}      
\usepackage{blkarray}
\usepackage{marvosym}

\usepackage{appendix}
\usepackage{hyperref}
\hypersetup{hidelinks}
\usepackage{natbib}
\usepackage{graphicx} 
\usepackage{subfigure}
\usepackage{longtable}
\usepackage{bm}

\usepackage{booktabs}
\usepackage{amsmath,amssymb,amsfonts}
\usepackage{multirow}
\newtheorem{theorem}{Theorem}

\newenvironment{proof}{\textit{Proof.}}{\hfill$\square$}

\DeclareMathOperator*{\argmin}{argmin}

\newcommand{\comment}[1]{}

\graphicspath{{figure/}}
\DeclareGraphicsExtensions{.pdf,.jpg,.eps,.png}

\usepackage{color}

\begin{document}
%
\title{Lifted Proximal Operator Machines}

\author{Jia Li\hspace{2em}Cong Fang\hspace{2em}Zhouchen Lin\Letter\\
Key Lab. of Machine Perception, School of EECS, Peking University\\
\texttt{jiali.gm@gmail.com fangcong@pku.edu.cn zlin@pku.edu.cn}\\
} \maketitle
\begin{abstract}
\begin{quote}
We propose a new optimization method for training feed-forward
neural networks. By rewriting the activation function as an
equivalent proximal operator, we approximate a feed-forward neural
network by adding the proximal operators to the objective function
as penalties, hence we call the lifted proximal operator machine
(LPOM). LPOM is \emph{block multi-convex} in all layer-wise
weights and activations. This allows us to use block coordinate
descent to update the layer-wise weights and activations \emph{in
parallel}. Most notably, we only use the mapping of the activation
function \emph{itself}, rather than its derivatives, thus avoiding
the gradient vanishing or blow-up issues in gradient based
training methods. So our method is applicable to various
non-decreasing Lipschitz continuous activation functions, which
\emph{can be saturating and non-differentiable}. LPOM does not
require more auxiliary variables than the layer-wise activations,
thus using roughly the same amount of memory as stochastic
gradient descent (SGD) does. We further prove the convergence of
updating the layer-wise weights and activations. Experiments on
MNIST and CIFAR-10 datasets testify to the advantages of LPOM.
\end{quote}
\end{abstract}


\section{Introduction}
Feed-forward deep neural networks (DNNs) are cascades of fully
connected layers and there are no feedback connections. In recent
years, with the advances in hardware and dataset sizes,
feed-forward DNNs have become standard in many tasks, such as
image recognition~\citep{krizhevsky2012imagenet}, speech
recognition~\citep{hinton2012deep}, natural language
understanding~\citep{collobert2011natural}, and as a building
block of the Go game learning system~\citep{silver2016mastering}.

For several decades, training a DNN is accomplished by optimizing
a highly nonconvex and nested function of the network weights. The
predominant method for training DNNs is stochastic gradient
descent (SGD)~\citep{rumelhart1986learning}, whose effectiveness
has been demonstrated by the successes of DNNs in various
real-world applications. Recently, many variants of SGD have been
proposed, which use adaptive learning rates and momentum terms,
e.g., Nesterov momentum~\citep{sutskever2013importance},
AdaGrad~\citep{duchi2011adaptive},
\mbox{RMSProp}~\citep{dauphin2015equilibrated}, and
Adam~\citep{kingma2014adam}. SGD and its variants use a few
training samples to estimate the full gradient, making the
computational complexity of each iteration small. Moreover, the
estimated gradients have noise, which is helpful for escaping
saddle points~\citep{ge2015escaping}. However, they have some
drawbacks as well. One major problem is the vanishing or blow-up
gradient issue, where the magnitudes of gradients decrease or
increase exponentially with the number of layers. This causes slow
or unstable convergence, especially in very deep networks. This
flaw can be remitted by using non-saturating activation functions,
such as rectified linear unit (ReLU), and modified network
architectures, such as ResNet~\citep{he2016deep}. However, the
fundamental problem remains. Furthermore, they cannot deal with
non-differentiable activation functions directly (e.g., binarized
neural networks~\citep{hubara2016binarized}) and do not allow
parallel weight updates across the layers. For more discussions on
the limitations of SGD, please refer
to~\citep{taylor2016training}.


The drawbacks of SGD motivate research on alternative training
methods for DNNs. Recently, training a feed-forward neural network
is formulated as a constrained optimization problem, where the
network activations are introduced as auxiliary variables and the
network configuration is guaranteed by layer-wise
constraints~\citep{carreira2014distributed}. It breaks the
dependency among the nested functions into equality constraints,
so many standard optimization methods can be utilized. Some
methods of this type of approach were studied and they differed in
how to handle the equality constraints.
\citet{carreira2014distributed} approximated the equality
constraints via quadratic penalties and alternately optimized
network weights and activations. \citet{zeng2018global} introduced
one more block of auxiliary variables per layer and also
approximated the equality constraints via quadratic penalties.
Inspired by alternating direction method of multiplier
(ADMM)~\citep{lin2011linearized}, \citet{taylor2016training} and
\citet{zhang2016efficient} used the augmented Lagrangian approach
to obtain exact enforcement of the equality constraints. However,
the two methods involved the Lagrange multipliers and nonlinear
constraints, thus were more memory demanding and more difficult in
optimization. Motivated by the fact that the ReLU activation
function is equivalent to a simple constrained convex minimization
problem, \citet{zhang2017convergent} relaxed the nonlinear
constraints as penalties, which encode the network architecture
and the ReLU activation function. Thus, the nonlinear constraints
no longer exist. However, their approach is limited to the ReLU
function and does not apply to other activation functions.
\citet{askari2018lifted} followed this idea by considering more
complex convex optimization problems and discussed several types
of non-decreasing activation functions. However, their methods to
update the weights and activations are still limited to the ReLU
function. Their approach cannot outperform SGD and can only serve
for producing good initialization for SGD. Actually, we have found
that their
formulation was incorrect (see Subsection ``Advantages of LPOM'').

This paper makes the following contributions:
\begin{itemize}
  \item[$\bullet$] We propose a new formulation to train feed-forward DNNs, which we call the lifted proximal operator machine (LPOM)\footnote{Patent filed.}. LPOM is block multi-convex, i.e., the problem is convex w.r.t. weights or activations of each layer when the remaining weights and activations are
  fixed. In contrast, almost all existing DNN training methods do
  not have such a property. This greatly facilitates the training
  of DNNs.
  \item[$\bullet$] Accordingly, we apply block coordinate descent (BCD) to solve LPOM, where the layer-wise weights and activations can be updated in parallel. Most notably, the update of the layer-wise weights or activations only utilizes the activation function itself, rather than its derivatives, thus avoiding the gradient vanishing or blow-up issues in gradient based training methods. Moreover, LPOM does not need more auxiliary variables than the layer-wise activations, thus its memory cost is close to that of SGD. We further prove that the iterations to update layer-wise weights or activations are convergent.
  \item[$\bullet$] Since only the activation function itself is involved in computation, LPOM is able to handle general non-decreasing\footnote{Although our theories can be easily extended to non-monotone activation functions, such a case is of less interest in reality.} Lipschitz continuous activation functions, which can be saturating (such as sigmoid and tanh) and non-differentiable (such as ReLU and leaky ReLU). So LPOM successfully overcomes the computation difficulties when using most of existing activation functions.
\end{itemize}
We implement LPOM on fully connected DNNs and test it on benchmark
datasets, MNIST and CIFAR-10, and obtain satisfactory results. For
convolutional neural networks (CNNs), since we have not
reformulated pooling and skip-connections, we leave the
implementation of LPOM on CNNs to future work. Note that the
existing non-gradient based approaches also focus on fully
connected DNNs
first~\citep{carreira2014distributed,zeng2018global,taylor2016training,zhang2016efficient,zhang2017convergent,askari2018lifted}.



\section{Related Work}
In a standard feed-forward neural network, the optimization
problem to train an $n$-layer neural network for classification
tasks is:
\begin{equation}\label{nested_problem}
\begin{split}
\hspace{-0.2cm} \min_{\{W^i\}}{\ell}\left(\phi
(W^{n-1}\phi(\cdots\phi (W^{2}\phi(W^{1}X^{1}))\cdots)), L\right),
\end{split}
\end{equation}
where $X^1 \in \mathbb{R}^{n_1\times m}$ is a batch of training
samples, $L \in \mathbb{R}^{c\times m}$ denotes the corresponding
labels, $n_1$ is the dimension of the training samples, $m$ is the
batch size, $c$ is the number of classes,
$\{W^i\}_{i\!=\!1}^{n\!-\!1}$ are the weights to be learned in
which the biases have been omitted for simplicity, $\phi(\cdot)$
is an element-wise activation function (e.g., sigmoid, tanh, and
ReLU), and ${\ell}(\cdot,\cdot)$ is the loss function (e.g., the
least-square error or the cross-entropy error). Here the neural
network is defined as a nested function, where the first layer
function of the neural network is $\phi(W^1X^1)$, the $i$-th layer
($i\!=\!2,\cdots,n$) function has the form $\phi(W^iX)$, and $X$
is the output of the $(i\!-\!1)$-th layer function. A common
approach to optimize~\eqref{nested_problem} is by SGD, i.e.,
calculating the gradient w.r.t. all weights of the network using
backpropagation and then updating the weights by gradient descent.

By introducing the layer-wise activations as a block of auxiliary
variables, the training of a neural network can be equivalently
formulated as an equality constrained optimization
problem~\citep{carreira2014distributed}:
\begin{equation}\label{problem}
\begin{split}
&\min_{\{W^i\}, \{X^i\}}{\ell}(X^n,L)\\
&~~\text{s.t.}~X^i\!=\!\phi(W^{i-1}X^{i\!-\!1}),~i\!=\!2,3,\cdots,n,
\end{split}
\end{equation}
where $X^i$ is the activation of the $i$-th layer and other
notations are the same as those in~\eqref{nested_problem}.
The constraints in~\eqref{problem} ensure that the auxiliary
variables $\{X^i\}_{i=2}^{n}$ exactly match the forward pass of
the network. Compared with problem~\eqref{nested_problem},
problem~\eqref{problem} is constrained. But since the objective
function is not nested, hence much simpler, such an equivalent
reformulation may lead to more flexible optimization methods. Note
that when using SGD to solve problem~\eqref{nested_problem}, it
actually works on problem~\eqref{problem} implicitly as \emph{the
activations $\{X^i\}_{i=2}^{n}$ need be recorded in order to
compute the gradient.}

Inspired by the quadratic-penalty method, \citet{carreira2014distributed} developed the method of auxiliary coordinates (MAC) to solve
problem~\eqref{problem}. MAC uses
quadratic penalties to approximately enforce equality constraints
and tries to solve the following problem:
\begin{equation}\label{problem_2split}
\begin{split}
\min_{\{W^i\},
\{X^i\}}\!{\ell}(X^n,L)\!+\!\frac{\mu}{2}\!\sum_{i=2}^n\|X^i\!-\!\phi(W^{i-1}X^{i\!-\!1})\|_F^2,
\end{split}
\end{equation}
where $\mu>0$ is a constant that controls the weight of the constraints and $\|\cdot\|_F$ is the Frobenius norm.
\citet{zeng2018global} decoupled the nonlinear activations in~\eqref{problem} with new auxiliary variables:
\begin{equation}\label{problem_3var}
\begin{split}
\hspace{-0.2cm}
&\min_{\{W^i\}, \{X^i\}, \{U^i\}}{\ell}(X^n,L)\\
\hspace{-0.2cm}
&\text{s.t.}~U^{i}\!=\!W^{i-1}X^{i\!-\!1}, X^i\!=\!\phi(U^{i}),~i\!=\!2,3,\cdots,n.
\end{split}
\end{equation}
This is called as the 3-splitting formulation. Accordingly,
problem~\eqref{problem} is the 2-splitting formulation. Following
the MAC method, rather than directly solving
problem~\eqref{problem_3var}, they optimized the following problem
instead:
\begin{equation}\label{problem_3split_Zeng}
\begin{split}
&\min_{\{W^i\}, \{X^i\}, \{U^i\}}{\ell}(X^n,L)\\
&+\!\frac{\mu}{2}\!\sum_{i=2}^n (\|U^{i}\!-\!W^{i-1}X^{i\!-\!1}\|_F^2\!+\!\|X^i\!-\!\phi(U^{i})\|_F^2).
\end{split}
\end{equation}
They adapted a BCD method to solve the above problem.

\begin{table*}
\caption{The $f(x)$ and $g(x)$ of several representative activation functions. Note that $0\!<\!\alpha\!<\!1$ for the leaky ReLU function and $\alpha\!>\!0$ for the exponential linear unit (ELU) function~\citep{clevert2015fast}. 
We only use $\phi(x)$ in our computation and do NOT explicitly use
$\phi^{-1}(x)$, $f(x)$, and $g(x)$. So all these activation
functions and many others can be used in LPOM. } \centering
\scalebox{0.95}[0.95]{
\begin{tabular}{p{1cm}||p{2.5cm}|p{2.9cm}|p{6.2cm}|l}
  \hline
   function & $\phi(x)$ & $\phi^{-1}(x)$ & $f(x)$ & $g(x)$ \\
  \hline\hline
  sigmoid & $\frac{1}{1\!+\!e^{-x}}$ &
  $\begin{array}{l}
  \log\frac{x}{1\!-\!x}\\
  (0<x<1)
  \end{array}$&
  $\left\{\!\!\begin{array}{rl}
  \!\!x\log x\!+\!(1\!-\!x)\log(1\!-\!x)\!-\!\frac{x^2}{2},&\!\!0\!<\!x\!<\!1\\
  +\infty,& \!\!\mbox{otherwise}
  \end{array} \right.$
  & $\log(e^x+1)\!-\!\frac{x^2}{2}$ \\
  \hline
  tanh & $\frac{e^x\!-\!e^{-x}}{e^x\!+\!e^{-x}}$ &
  $\begin{array}{l}
  \frac{1}{2}\log \frac{1\!+\!x}{1\!-\!x}\\
  (-1<x<1)
  \end{array}$
  &
  $\left\{\begin{array}{rl}
  \!\!\frac{1}{2}[(1\!-\!x)\log(1\!-\!x)\qquad\quad&\\
  \!\!+\!(1\!+x\!)\log(1\!+\!x))]\!-\!\frac{x^2}{2},&\!\!-1<x<1\\
  +\infty,& \!\!\mbox{otherwise}
  \end{array} \right.$
  & $\log(\frac{e^x\!+\!e^{-x}}{2})\!-\!\frac{x^2}{2}$  \\
  \hline
  ReLU & $\max(x,0)$ &
  $\left\{\begin{array}{rl}
  \!\!x,&\!\!x\!>\!0\\
  \!\!(-\infty,0),&\!\!x\!=\!0
  \end{array} \right.$
  &
  $\left\{\begin{array}{rl}
  \!\!0,&\!\!x\!\geq\!0\\
  \!\!+\infty,& \!\!\mbox{otherwise}
  \end{array} \right.$
  & $\left\{\begin{array}{rl}
\!\!0, & \!\!x\!\geq\!0 \\
\!\!-\frac{1}{2}x^2, &\!\!x\!<\!0 \\
\end{array} \right.$\\
  \hline
  leaky\mbox{~~~} ReLU & $\left\{\begin{array}{rl}
\!\!x, &\!\! x\!\geq\!0 \\
\!\!\alpha x, &\!\!x\!<\!0 \\
\end{array} \right.$ & $\left\{\begin{array}{rl}
\!\!x, & \!\!x\!\geq\!0 \\
\!\!x/\alpha, &\!\!x\!<\!0 \\
\end{array} \right.$ & $\left\{\begin{array}{rl}
\!\!0, & \!\!x\!\geq\!0 \\
\!\!\frac{1\!-\!\alpha}{2\alpha}x^2, &\!\!x\!<\!0 \\
\end{array} \right.$ & $\left\{\begin{array}{rl}
\!\!0, & \!\!x\!\geq\!0 \\
\!\!\frac{\alpha\!-\!1}{2}x^2, &\!\!x\!<\!0 \\
\end{array} \right.$\\
  \hline
  ELU & $\left\{\!\!\!\begin{array}{lr}
x, & \!\! x\!\geq\!0 \\
\alpha(e^x\!-\!1), & \!\! x\!<\!0 \\
\end{array} \right.$ & $\left\{\!\!\!\begin{array}{lr}
x, & \! x\!\geq\!0 \\
\log(1\!+\!\frac{x}{\alpha}), & \! x\!<\!0 \\
\end{array} \right.$ &  $\left\{\!\!\! \begin{array}{lr}
0, & \! x\!\geq\!0 \\
(\alpha\!+\!x)(\log(\frac{x}{\alpha}\!+\!1)\!-\!1)\!-\!\frac{x^2}{2}, & \! x\!<\!0 \\
\end{array} \right.$ & $\left\{\!\!\! \begin{array}{lr}
0, & \! x\!\geq\!0 \\
\alpha(e^x\!-\!x)\!-\!\frac{x^2}{2}, & \! x\!<\!0 \\
\end{array} \right.$\\
\hline
softplus & $\log(1\!+\!e^x)$ & $\log(e^x\!-\!1)$ & No analytic expression & No analytic expression \\ 
\hline
\end{tabular}
}
\label{tab:demo_functions}
\end{table*}

\citet{taylor2016training} also considered solving
problem~\eqref{problem_3var}. Inspired
by ADMM~\citep{lin2011linearized}, they added a Lagrange
multiplier to the output layer to achieve exact enforcement of the
equality constraint at the output layer, which yields
\begin{equation}\label{problem_3split_Taylor}
\begin{split}
\hspace{-0.5cm}
&\min_{\{W^i\},\{X^i\},\{U^i\},M}{\ell}(U^n,L)\!+\!\frac{\beta}{2}\!\left\|U^{n}\!-\!W^{n-1}X^{n\!-\!1}\!+\!M\right\|_F^2\\
\hspace{-0.5cm}
&~~+\!\!\sum_{i=2}^{n-1}\frac{\mu_i}{2} (\|U^{i}\!-\!W^{i-1}X^{i\!-\!1}\|_F^2\!+\!\|X^i\!-\!\phi(U^{i})\|_F^2),
\end{split}
\end{equation}
where $M$ is the Lagrange multiplier and $\beta>0$ and $\mu_i>0$
are constants. Note that the activation function on the output
layer is absent. So \eqref{problem_3split_Taylor} is only a
heuristic adaptation of ADMM. \citet{zhang2016efficient} adopted a
similar technique but used a different variable splitting scheme:
\begin{equation}\label{problem_3var_zhang}
\begin{split}
\hspace{-0.2cm}
&\min_{\{W^i\}, \{X^i\}, \{U^i\}}{\ell}(X^n,L)\\
\hspace{-0.2cm}
&\text{s.t.}~U^{i-1}\!=\!X^{i-1}, X^i\!=\!\phi(W^{i\!-\!1}U^{i\!-\!1}),~i\!=\!2,3,\cdots,n.
\end{split}
\end{equation}
Despite the nonlinear equality constraints, which ADMM is not
designed to handle, they added a Lagrange multiplier for each
constraint in~\eqref{problem_3var_zhang}. Then the augmented
Lagrangian problem is as follows:
\begin{equation}\label{problem_3split_zhang}
\begin{split}
&\min_{\{W^i\}, \{X^i\}, \{U^i\}, \{A^i\}, \{B^i\}}{\ell}(X^n,L)\\
&\quad+\!\frac{\mu}{2}\!\sum_{i=2}^{n} \left(\left\|U^{i-1}\!-\!X^{i-1}\!+\!A^{i-1}\right\|_F^2\right.\\
&\qquad+\!\left.\left\|X^i\!-\!\phi(W^{i-1}U^{i\!-\!1})\!+\!B^{i-1}\right\|_F^2\right),
\end{split}
\end{equation}
where $A^i$ and $B^i$ are the Lagrange multipliers.

Different from naively applying the penalty method and ADMM, \citet{zhang2017convergent} interpreted the ReLU activation function as a simple
smooth convex optimization problem.
Namely, the equality constraints in problem~\eqref{problem} using
the ReLU activation function can be rewritten as a convex
minimization problem:
\begin{equation}
\begin{split}
X^i&\!=\!\phi(W^{i-1}X^{i\!-\!1})\\
   &\!=\!\max(W^{i-1}X^{i\!-\!1}, \textbf{0})\\
   &\!=\!\argmin\limits_{U^i\geq \textbf{0}}\|U^i-W^{i-1}X^{i\!-\!1}\|_F^2,
\end{split}
\end{equation}
where $\textbf{0}$ is a zero matrix with an appropriate size.
Based on this observation, they approximated problem~\eqref{problem}
with the activation function being ReLU in the following way:
\begin{equation}\label{problem_lifted}
\begin{split}
&\min_{\{W^i\}, \{X^i\}}\!{\ell}(X^n,L)\!+\!\sum_{i=2}^n\frac{\mu_i}{2}\|X^i\!-\!W^{i-1}X^{i\!-\!1}\|_F^2\\
&~~\text{s.t.}~X^i\geq \textbf{0},~i\!=\!2,3,\cdots,n,
\end{split}
\end{equation}
where the penalty terms encode both the network structure and
activation function. Unlike MAC and ADMM based methods, it does
not include nonlinear activations. Moreover, the major advantage
is that problem~\eqref{problem_lifted} is block multi-convex,
i.e., the problem is convex w.r.t. each block of variables when
the remaining blocks are fixed. They developed a new BCD method to
solve it. They also empirically demonstrated the superiority of
the proposed approach over SGD based solvers in
Caffe~\citep{jia2014caffe} and the \mbox{ADMM} based
method~\citep{zhang2016efficient}. \citet{askari2018lifted}
inherited the same idea. By introducing a more complex convex
minimization problem, they could handle more general activation
functions, such as sigmoid, leaky ReLU, and sine functions.

\section{Lifted Proximal Operator Machine}
In this section, we describe our basic idea of LPOM and its advantages over existing DNN training methods. 

\subsection{Reformulation by Proximal Operator}
We assume that the activation function $\phi$ is non-decreasing.
Then $\phi^{-1}(x)=\{y|x=\phi(y)\}$ is a convex set.
$\phi^{-1}(x)$ is a singleton $\{y\}$ iff $\phi$ is strictly
increasing at $\phi(y)$. We want to construct an objective
function $h(x,y)$, parameterized by $y$, such that its minimizer
is exactly $x=\phi(y)$. Accordingly, we may replace the constraint
$x=\phi(y)$ by minimizing $h(x,y)$, which can be added to the loss
of DNNs as a penalty.

There are two basic operations to update variables in an
optimization problem: gradient update and proximal operator. Since
we are constructing an optimization problem and proximal operator
is indeed so~\citep{parikh2014proximal}:
\begin{equation}\label{prox_function}
\text{prox}_f(y)\!=\!\argmin_x f(x)\!+\!\frac{1}{2}(x-y)^2,
\end{equation}
we consider using proximal operator to construct the optimization
problem. Define $$f(x)\!= \!\int_0^x (\phi^{-1}(y)\!-\!y)dy.$$
Note that $f(x)$ is well defined\footnote{We allow $f$ to take
value of $+\infty$.} even if $\phi^{-1}(y)$ is non-unique for some
$y$ between 0 and $x$. Anyway, $\phi^{-1}$, $f$, and $g$ (to be
defined later) will \emph{not} be explicitly used in our
computation. It is easy to show that the optimality condition of
\eqref{prox_function} is $0\in (\phi^{-1}(x)-x) + (x-y)$. So the
solution to \eqref{prox_function} is exactly $x=\phi(y)$.

Note that $f(x)$ is a unit-variate function. For a matrix
$X\!=\!(X_{kl})$, we define $f(X)\!=\!(f(X_{kl}))$. Then the
optimality condition of the following minimization problem:
\begin{equation}\label{matrix_prox}
\argmin_{X^i}
\mathbf{1}^Tf(X^i)\mathbf{1}\!+\!\frac{1}{2}\|X^i\!-\!W^{i-1}X^{i-1}\|^2_F,
\end{equation}
where $\mathbf{1}$ is an all-one column vector, is
\begin{equation}
\mathbf{0}\in \phi^{-1}(X^i)-W^{i-1}X^{i-1},
\end{equation}
where $\phi^{-1}(X^i)$ is also defined element-wise. So the
optimal solution to \eqref{matrix_prox} is
\begin{equation}\label{equality_constraint}
X^i\!=\!\phi(W^{i-1}X^{i-1}),
\end{equation}
which is exactly the constraint in problem~\eqref{problem}. So we
may approximate problem~\eqref{problem} naively as:
\begin{equation}\label{problem_f_only}
\begin{split}
\hspace{-0.2cm}
&\min\limits_{\{W^i\},\{X^i\}} {\ell}(X^n,L)\\
&\quad+\sum\limits_{i=2}^n
\mu_i\left(\mathbf{1}^Tf(X^i)\mathbf{1}\!+\!\frac{1}{2}\|X^{i}\!-\!W^{i-1}X^{i-1}\|_F^2\right).
\end{split}
\end{equation}
However, its optimality conditions for $\{X^i\}_{i=2}^{n-1}$ are
as follows:
\begin{equation}\label{incorrect_optimality}
\begin{split}
\mathbf{0}\!\in&\mu_i(\phi^{-1}(X^i)\!-\!W^{i-1}X^{i-1})\\
&+\!\mu_{i+1}(W^i)^T(W^iX^i\!-\!X^{i+1}),\,i=2,\cdots,n-1.
\end{split}
\end{equation}
We can clearly see that the equality constraints
\eqref{equality_constraint} in problem~\eqref{problem} do
\emph{not} satisfy the above!

In order that the equality constraints \eqref{equality_constraint}
fulfil the optimality conditions of approximating problem, we need
to modify \eqref{incorrect_optimality} as
\begin{equation}\label{correct_optimality}
\begin{split}
\mathbf{0}\!\in&\mu_i(\phi^{-1}(X^i)\!-\!W^{i-1}X^{i-1})\\
&+\!\mu_{i+1}(W^i)^T(\phi(W^iX^i)\!-\!X^{i+1}),\,i=2,\cdots,n-1.
\end{split}
\end{equation}
This corresponds to the following problem:
\begin{equation}\label{problem_LPOM}
\begin{split}
\hspace{-0.2cm}
&\min\limits_{\{W^i\},\{X^i\}} {\ell}(X^n,L)\!+\!\sum\limits_{i=2}^n \mu_i\bigg(\mathbf{1}^Tf(X^i)\mathbf{1}\\
\hspace{-0.2cm}
&~~~~~~\left.\!+\mathbf{1}^Tg(W^{i-1}\!X^{i-1})\mathbf{1}\!+\!\frac{1}{2}\|X^{i}\!-\!W^{i-1}X^{i-1}\|_F^2\right),
\end{split}
\end{equation}
where $$g(x)\!= \!\int_0^x (\phi(y)\!-\!y)dy.$$ $g(X)$ is also
defined element-wise for a matrix $X$. The $f(x)$'s and $g(x)$'s
of some representative activation functions are shown in
Table~\ref{tab:demo_functions}. \eqref{problem_LPOM} is the
formulation of our proposed LPOM, where we highlight that the
introduction of $g$ is non-trivial and non-obvious.

\subsection{Advantages of LPOM}
Denote the objective function of LPOM  in~\eqref{problem_LPOM} as
$F(W,X)$. Then we have the following theorem:
\begin{theorem}\label{convex_theorem}
Suppose ${\ell}(X^n, L)$ is convex in $X^n$ and $\phi$ is
non-decreasing. Then $F(W,X)$ is block multi-convex, i.e., convex
in each $X^i$ and $W^i$ if all other blocks of variables are
fixed.
\end{theorem}
\begin{proof}
    $F(W,X)$ can be simplified to
    \begin{equation}\label{problem_LPOM_simplfied}
    \begin{split}
    \hspace{-0.2cm}
    &F(W,X)={\ell}(X^n,L)\!+\!\sum\limits_{i=2}^n \mu_i\left(\mathbf{1}^T\tilde{f}(X^i)\mathbf{1}\frac{}{}\right.\\
    \hspace{-0.2cm}
    &~~~~~~\!+\mathbf{1}^T\tilde{g}(W^{i-1}\!X^{i-1})\mathbf{1}\!-\!\langle X^{i},W^{i-1}X^{i-1}\rangle\Big),
    \end{split}
    \end{equation}
where $\tilde{f}(x)\!=\!\int_0^x \phi^{-1}(y) dy$ and
$\tilde{g}(x)\!=\!\int_0^x \phi(y) dy$. Since both $\phi$ and
$\phi^{-1}$ are non-decreasing, both $\tilde{f}(x)$ and
$\tilde{g}(x)$ are convex. It is easy to verify that
$\mathbf{1}^T\tilde{g}(W^{i-1}\!X^{i-1})\mathbf{1}$ is convex in
$X^{i-1}$ when $W^{i-1}$ is fixed and convex in $W^{i-1}$ when
$X^{i-1}$ is fixed. The remaining term $\langle
X^{i},W^{i-1}X^{i-1}\rangle$ in $F(W,X)$ is linear in one block
when the other two blocks are fixed. The proof is completed.
\end{proof}

Theorem~\ref{convex_theorem} allows for efficient BCD algorithms
to solve LPOM and guarantees that the optimal solutions for
updating $X^i$ and $W^i$ can be obtained, due to the convexity of
subproblems. In contrast, the subproblems in the penalty and the
ADMM based methods are all nonconvex.

When compared with ADMM based
methods~\citep{taylor2016training,zhang2016efficient}, LPOM does
not require Lagrange multipliers and more auxiliary variables than
$\{X^i\}_{i=2}^{n}$. Moreover, we have designed delicate
algorithms so that no auxiliary variables are needed either when
solving LPOM (see Section ``Solving LPOM''). So LPOM has much less
variables than ADMM based methods and hence saves memory greatly.
Actually, its memory cost is close to that of SGD\footnote{As
noted before, SGD needs to save $\{X^i\}_{i=2}^n$, although it
does not treat $\{X^i\}_{i=2}^n$ as auxiliary variables.}.


When compared with the penalty
methods~\citep{carreira2014distributed,zeng2018global}, the
optimality conditions of LPOM are simpler. For example, the
optimality conditions for $\{X^i\}_{i=2}^{n-1}$ and
$\{W^i\}_{i=1}^{n-1}$ in LPOM are \eqref{correct_optimality} and
\begin{equation}\label{KKT_problem_LPOM_W}
\begin{split}
(\phi(W^{i}X^{i})\!-\!X^{i+1})(X^{i})^T \!=\! \mathbf{0}, \quad
i\!=\!1,\cdots,n-1,
\end{split}
\end{equation}
while those for MAC are
\begin{equation}\label{KKT_problem_penalty}
\begin{split}
\hspace{-0.2cm}
&(X^i\!-\!\phi(W^{i-1}X^{i-1})) \\
\hspace{-0.2cm}
&+(W^{i})^T[(\phi(W^iX^i)\!-\!X^{i+1})\!\circ\!\phi'(W^iX^i)] \!=\! \mathbf{0},\\
\hspace{-0.2cm} & i\!=\!2,\cdots,n\!-\!1.
\end{split}
\end{equation}
and
\begin{equation}\label{KKT_problem_penalty_W}
\begin{split}
[(\phi(W^{i}X^{i})\!-\!X^{i+1})\!\circ\!\phi'(W^iX^i)](X^{i})^T\!=\!\mathbf{0},
i\!=\!1,\cdots,n\!-\!1,
\end{split}
\end{equation}
where $\circ$ denotes the element-wise multiplication. We can see
that the optimality conditions for MAC have extra $\phi'(W^iX^i)$,
which is nonlinear. The optimality conditions for
\cite{zeng2018global} can be found in Supplementary Materials.
They also have an extra $\phi'(U^i)$. This may imply that the
solution sets of MAC and \citep{zeng2018global} are more complex
and also ``larger'' than that of LPOM. So it may be easier to find
good solutions of LPOM.

When compared with the convex optimization reformulation
methods~\citep{zhang2017convergent,askari2018lifted}, LPOM can
handle much more general activation functions. Note that
\cite{zhang2017convergent} only considered ReLU. Although
\cite{askari2018lifted} claimed that their formulation can handle
general activation functions, its solution method was still
restricted to ReLU. Moreover, \citet{askari2018lifted} do not
have a correct reformulation as its optimality conditions for
$\{X^i\}_{i=2}^{n-1}$ and $\{W^i\}_{i=1}^{n-1}$ are
\begin{equation*}\label{KKT_problem_lifted_NN_X}
\begin{split}
\mathbf{0}\in&\mu_i(\phi^{-1}(X^{i})\!-\!W^{i-1}X^{i-1})\!-\!\mu_{i+1}(W^i)^TX^{i+1},\\
&i\!=\!2,\cdots,n\!-\!1,
\end{split}
\end{equation*}
and $X^{i+1}(X^i)^T\!=\!\mathbf{0},\,i\!=\!1,\cdots,n\!-\!1,$
respectively. It is clear that the equality
constraints~\eqref{equality_constraint} do not satisfy the above. Moreover, somehow \cite{askari2018lifted}
further added extra constraints $X^i\geq \mathbf{0}$, no matter
what the activation function is. So their reformulation cannot
approximate the original DNN \eqref{problem} well. This may
explain why \cite{askari2018lifted} could not obtain good results.
Actually, they can only provide good initialization for SGD.

When compared with gradient based methods, such as SGD, LPOM can
work with any non-decreasing Lipschitz continuous activation
function without numerical difficulties\footnote{Of course, for
different choices of activation functions the performances of DNNs
may differ.}, including being saturating (e.g., sigmoid and tanh)
and non-differentiable (e.g., ReLU and leaky ReLU) and can update
the layer-wise weights and activations \emph{in parallel} (see
next section)\footnote{But our current implementation for the
experiments is still serial.}. In contrast, gradient based methods
can only work with limited activation functions, such as ReLU,
leaky ReLU, and softplus, in order to avoid the gradient vanishing
or blow-up issues, and they cannot be parallelized when computing
the gradient and the activations.



\section{Solving LPOM}
Thanks to the block multi-convexity
(Theorem~\ref{convex_theorem}), LPOM can be solved by BCD. Namely,
we update $X^i$ or $W^i$ by fixing all other blocks of variables.
The optimization can be performed using a mini-batch of training
samples. The whole algorithm to solve LPOM is summarized in
Algorithm~\ref{alg:solve_LPOM}. Below we give more details.

\begin{algorithm}[tb]
    \begin{algorithmic}
     \REQUIRE training dataset, batch size $m_1$, iteration no.s $S$ and $K_1$.
   \STATE  {\textbf{for} $s=1$ to $S$ \textbf{do}}
   \STATE $\quad$ Randomly choose $m_1$ training samples $X^1$ and $L$.
   \STATE $\quad$ Solve $\{X^{i}\}_{i=2}^{n-1}$ by iterating  Eq.~\eqref{eq:update_Xi} for $K_1$ times (or until convergence).
   \STATE $\quad$ Solve $X^{n}$ by iterating Eq.~\eqref{solution_Xn_iteration} for $K_1$ times.
   \STATE $\quad$ Solve $\{W^i\}_{i=1}^{n-1}$ by applying Algorithm~\ref{alg:solve_x} to \eqref{problem_update_W}.
     \STATE  {\textbf{end for}}
    \ENSURE  $\{W^i\}_{i=1}^{n-1}$.
    \end{algorithmic}
    \caption{Solving LPOM}
    \label{alg:solve_LPOM}
\end{algorithm}


\subsection{Updating $\{X^i\}_{i=2}^n$}
We first introduce the serial method for updating
$\{X^i\}_{i=2}^n$. We update $\{X^i\}_{i=2}^{n}$ from $i=2$ to $n$
successively, just like the feed-forward process of DNNs. For
$i=2,\cdots,n-1$, with $\{W^i\}_{i=1}^{n-1}$ and other
$\{X^j\}_{j=2,j\neq i}^n$ fixed, problem \eqref{problem_LPOM}
reduces to
\begin{equation}\label{problem_Xi}
\begin{split}
&\min\limits_{X^i} \mu_i \left(\mathbf{1}^Tf(X^i)\mathbf{1}\!+\!\frac{1}{2}\|X^{i}\!-\!W^{i-1}X^{i-1}\|_F^2 \right)\\
&+\!\mu_{i+1}\left(\mathbf{1}^Tg(W^{i}\!X^{i})\mathbf{1}\!+\!\frac{1}{2}\|X^{i+1}\!-\!W^{i}X^{i}\|_F^2\right).
\end{split}
\end{equation}
The optimality condition is:
\begin{equation}
\begin{split}
\mathbf{0}\in &\mu_i(\phi^{-1}(X^i)\!-\!W^{i-1}X^{i-1}) \\
&+\mu_{i+1}((W^{i})^T(\phi(W^iX^i)\!-\!X^{i+1})).
\end{split}
\end{equation}
So we may update $X^i$ by iterating
\begin{equation}\label{eq:update_Xi}
\begin{split}
X^{i,t+1}\!=\!\phi\!\left(W^{i\!-\!1}X^{i\!-\!1}\!-\!\frac{\mu_{i\!+\!1}}{\mu_i}(W^{i})^T(\phi(W^iX^{i,t})\!-\!X^{i\!+\!1})\right)
\end{split}
\end{equation}
until convergence, where the superscript $t$ is the iteration
number. The convergence analysis is as follows:
\begin{theorem}\label{theorem_solve_Xi}
Suppose that $|\phi'(x)|\leq\gamma$.\footnote{For brevity, we
assume that $\phi$ is differentiable. It is not hard to extend our
analysis to the non-differentiable case, but the analysis will be
much more involved.} If $\rho<1$, then the iteration is convergent
and the convergent rate is linear, where
$\rho\!=\!\frac{\mu_{i+1}}{\mu_i}\gamma^2\sqrt{\left\|
|(W^i)^T||W^i|\right\|_1\left\| |(W^i)^T| |W^i|\right\|_\infty} $.
\end{theorem}
The proof can be found in Supplementary Materials. In the above,
$|A|$ is a matrix whose entries are the absolute values of $A$,
$\|\cdot\|_1$ and $\|\cdot\|_{\infty}$ are the matrix 1-norm
(largest absolute column sum) and the matrix $\infty$-norm
(largest absolute row sum), respectively.

When considering $X^n$, problem~\eqref{problem_LPOM} reduces to
\begin{equation}\label{problem_Xn}
\begin{split}
\min\limits_{X^n}
{\ell}(X^n,L)\!+\!\mu_n\!\left(\!\mathbf{1}^Tf(X^i)\mathbf{1}\!+\!\frac{1}{2}\|X^{n}\!-\!W^{n\!-\!1}X^{n\!-\!1}\|_F^2\!\right).
\end{split}
\end{equation}
The optimality condition is\footnote{For simplicity, we also
assume that the loss function is differentiable w.r.t. $X^n$.}
\begin{equation}\label{solution_Xn}
\begin{split}
\mathbf{0}\in\frac{\partial {\ell}(X^n,L)}{\partial
X^n}\!+\!\mu_n(\phi^{-1}(X^n)\!-\!W^{n-1}X^{n-1}).
\end{split}
\end{equation}
So we may update $X^n$ by iterating
\begin{equation}\label{solution_Xn_iteration}
\begin{split}
X^{n,t+1}\!=\!\phi\left(W^{n-1}X^{n\!-\!1}\!-\!\frac{1}{\mu_n}\frac{\partial
{\ell}(X^{n,t},L)}{\partial X^n}\right)
\end{split}
\end{equation}
until convergence. The convergence analysis is as follows:
\begin{theorem}\label{theorem_solve_Xn}
Suppose that $|\phi'(x)|\!\leq\!\gamma$ and $\left\|
\left|\left(\frac{\partial^2 {\ell(X,L)} }{\partial X_{kl}\partial
X_{pq}} \right)\right| \right\|_1 \!\leq\! \eta$. If $\tau \!<\!
1$, then the iteration is convergent and the convergent rate is
linear, where $\tau\!=\!\frac{\gamma\eta}{\mu_n}$.
\end{theorem}
The proof can also be found in Supplementary Materials. If
${\ell}(X^n,L)$ is the least-square error, i.e.,
${\ell}(X^n,L)\!=\!\frac{1}{2}\|X^n\!-\!L\|_F^2$, then $\left\|
\left|\left(\frac{\partial^2 {\ell(X,L)} }{\partial X_{kl}\partial
X_{pq}}\right) \right| \right\|_1 \!=\! 1$. So we obtain
$\mu_n\!>\!\gamma$.

The above serial update procedure can be easily changed to
parallel update: each $X^i$ is updated using the latest
information of other $X^j$'s, $j\neq i$.



\subsection{Updating $\{W^i\}_{i=1}^{n-1}$}
$\{W^i\}_{i=1}^{n-1}$ can be updated with full parallelism. When
$\{X^i\}_{i=2}^{n}$ are fixed, problem \eqref{problem_LPOM}
reduces to
\begin{equation}\label{problem_update_W_org}
\min\limits_{W^i}
\mathbf{1}^Tg(W^{i}X^{i})\mathbf{1}\!+\!\frac{1}{2}\|W^{i}X^{i}\!-\!X^{i+1}\|_F^2,\,i=1,\cdots,n-1,
\end{equation}
which can be solved in parallel. \eqref{problem_update_W_org} can
be rewritten as
\begin{equation}\label{problem_update_W}
\min\limits_{W^i}
\mathbf{1}^T\tilde{g}(W^{i}X^{i})\mathbf{1}\!-\!\langle
X^{i+1},W^{i}X^{i}\rangle,
\end{equation}
where $\tilde{g}(x)\!=\!\int_0^x \phi(y)dy$, as introduced before.
Suppose that $\phi(x)$ is $\beta$-Lipschitz continuous, which is
true for almost all activation functions in use. Then
$\tilde{g}(x)$ is $\beta$-smooth:
\begin{equation}
\begin{split}
|\tilde{g}'(x)\!-\!\tilde{g}'(y)|\!=\!|\phi(x)\!-\!\phi(y)|\leq
\beta|x\!-\!y|.
\end{split}
\end{equation}

Problem \eqref{problem_update_W} could be solved by
APG~\citep{Beck2009FISTA} by locally linearizing
$\hat{g}(W)=\tilde{g}(WX)$. However, the Lipschitz constant of the
gradient of $\hat{g}(W)$, which is $\beta\|X\|_2^2$, can be very
large, hence the convergence can be slow. Below we propose an
improved version of APG that is tailored for solving
\eqref{problem_update_W} much more efficiently.

Consider the following problem:
\begin{equation}\label{eq:original_problem}
\min_x F(x)\!\equiv\!\varphi(Ax)+h(x),
\end{equation}
where both $\varphi(y)$ and $h(x)$ are convex. Moreover,
$\varphi(y)$ is $L_\varphi$-smooth: $\|\nabla
\varphi(x)\!-\!\nabla \varphi(y)\| \!\leq \!L_\varphi\|x\!-\!y\|,
\forall x,y.$ We assume that the following problem
\begin{equation}\label{eq:update_x}
x_{k+1}\!=\!\argmin_x \langle \nabla
\varphi(Ay_k),A(x\!-\!y_k)\rangle\!+\!\frac{L_\varphi}{2}\|A(x\!-\!y_k)\|^2\!+\!h(x)
\end{equation}
is easy to solve for any given $y_k$. We propose
Algorithm~\ref{alg:solve_x} to solve \eqref{eq:original_problem}.
Then we have the following theorem:
\begin{theorem}
If we use Algorithm~\ref{alg:solve_x} to solve
problem~\eqref{eq:original_problem}, then the convergence rate is
at least $O(k^{-2})$:
$$
F(x_k)\!-\!F(x^*)\!+\!\frac{L_\varphi}{2}\|z_k\|^2\!\leq\!\frac{4}{k^2}\!\left(\!F(x_1)\!-\!F(x^*)\!+\!\frac{L_\varphi}{2}\|z_1\|^2\!\right),
$$
where
$z_k\!=\!A[\theta_{k-1}x_{k-1}\!-\!x_k\!+\!(1\!-\!\theta_{k-1})x^*]$
and $x^*$ is any optimal solution to
problem~\eqref{eq:original_problem}.
\end{theorem}
The proof can also be found in Supplementary Materials.

\begin{algorithm}[tb]
    \begin{algorithmic}
    \REQUIRE $x_0$, $x_1$, $\theta_0\!=\!0$, $k=1$, iteration no. $K_2$.
    \STATE  {\textbf{for} $k=1$ to $K_2$ \textbf{do}}
    \STATE $\quad$ Compute $\theta_k$ via $1\!-\!\theta_k\!=\!\sqrt{\theta_k}(1\!-\!\theta_{k-1})$.
    \STATE $\quad$ Compute $y_k$ via $y_k\!=\!\theta_kx_k\!-\!\sqrt{\theta_k}(\theta_{k-1}x_{k-1}\!-\!x_k)$.
    \STATE $\quad$ Update $x_{k+1}$ via~\eqref{eq:update_x}.
    \STATE  {\textbf{end for}}
    \ENSURE  $x_k$.
    \end{algorithmic}
    \caption{Solving \eqref{eq:original_problem}.}
    \label{alg:solve_x}
\end{algorithm}

By instantiating problem~\eqref{eq:original_problem} with problem
\eqref{problem_update_W}, subproblem~\eqref{eq:update_x} becomes
\begin{equation}
\begin{split}
W^{i,t+1}=&\argmin_{W} \left\langle\phi(Y^{i,t}X^{i}), (W\!-\!Y^{i,t})X^i\right\rangle\\
&\!+\!\frac{\beta}{2}\| (W\!-\!Y^{i,t})X^i\|_F^2\!-\!\langle
X^{i+1}, WX^{i}\rangle.
\end{split}
\end{equation}
It is a least-square problem and the solution is:
\begin{equation}\label{eq:update_Wi}
W^{i,t+1}\!=\!\!Y^{i,t}\!-\!\frac{1}{\beta}(\phi(Y^{i,t}X^i)\!-\!X^{i+1})(X^i)^\dag,
\end{equation}
where $(X^i)^\dag$ is the pseudo-inverse of $X^i$ and $Y^{i,t}$
plays the role of $y_k$ in Algorithm~\ref{alg:solve_x}.

If $\phi(x)$ is strictly increasing and the rate of increment
$\frac{\phi(y)-\phi(x)}{y-x}\, (y\neq x)$ is lower bounded by
$\alpha>0$, then $\tilde{g}(x)$ is strongly convex and the
convergence is linear~\citep{Nesterov2000Introductory}. We omit
the details here.

\section{Experiments}

\begin{figure*}[!t]
    \hspace{-0.8em}
    \begin{tabular}{c@{\extracolsep{0.2em}}c@{\extracolsep{0.2em}}c@{\extracolsep{0.2em}}c}
        \vspace{-0.5em}
        \includegraphics[width=0.22\textwidth]{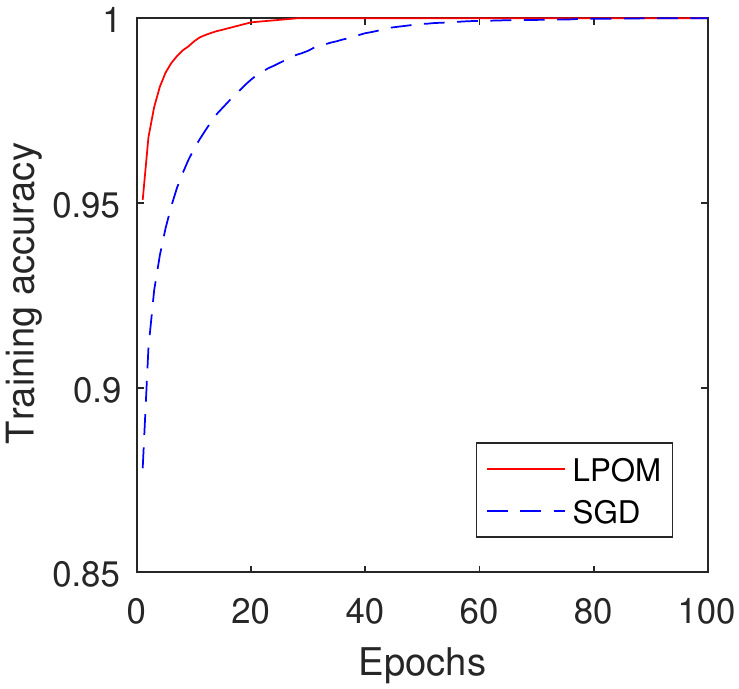}&
        \includegraphics[width=0.22\textwidth]{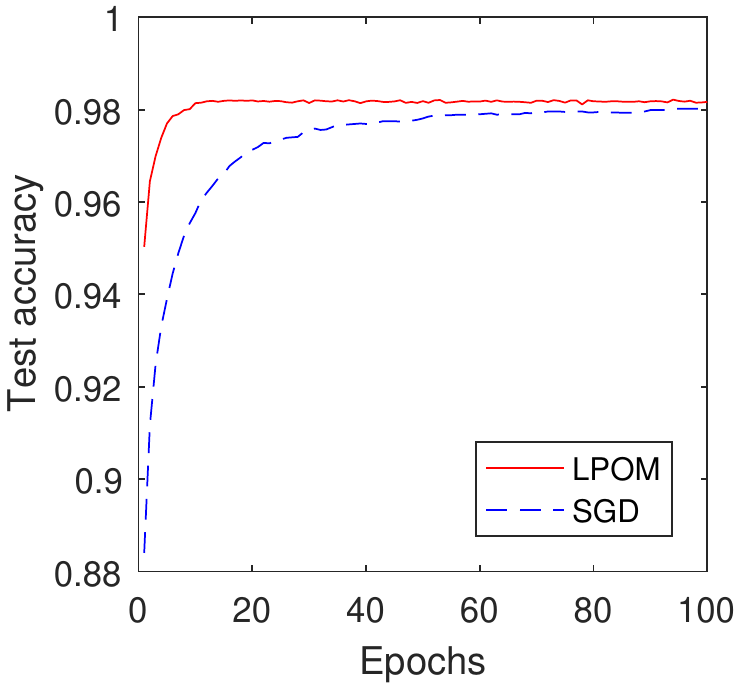}&
        \includegraphics[width=0.21\textwidth]{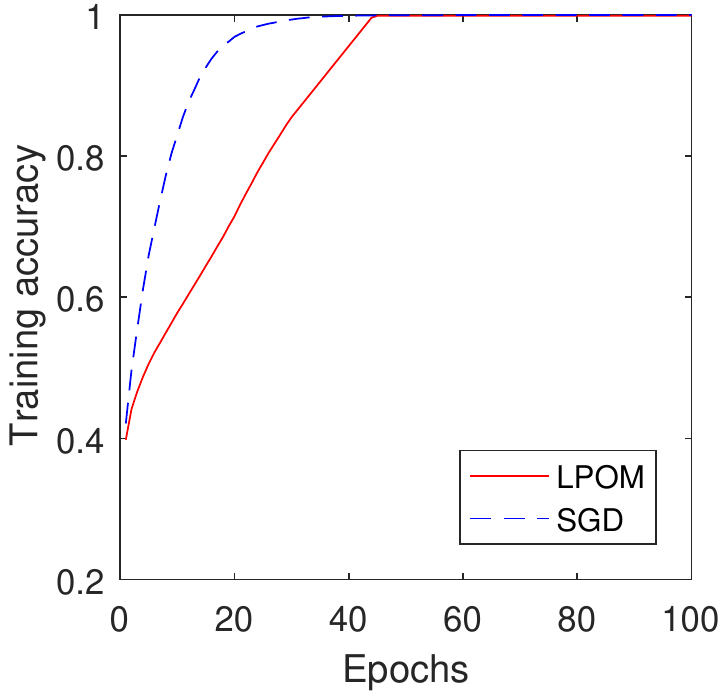}&
        \includegraphics[width=0.22\textwidth]{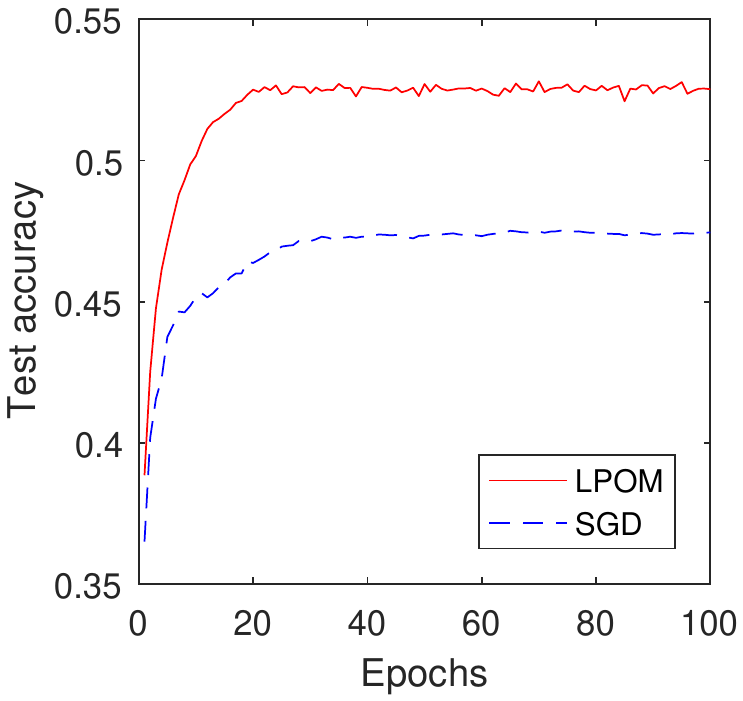}\\
        \vspace{-0.2em}
        {\scriptsize (a) MNIST (Training Acc.) }
        & {\scriptsize (b) MNIST (Test Acc.)}
        & {\scriptsize (c) CIFAR-10 (Training Acc.)}
        & {\scriptsize (d) CIFAR-10 (Test Acc.)}
    \end{tabular}
    \caption{Comparison of LPOM and SGD on the MNIST and the CIFAR-10 datasets. }
    \label{fig:comparison}
\end{figure*}


\begin{table*}
\begin{floatrow}
\scalebox{0.93}[0.93]{
\capbtabbox{
 \begin{tabular}{c||c|c|c|c|c}
        \hline
        Hidden layers & 300 & 300-100 & 500-150 & 500-200-100 & 400-200-100-50\\
        \hline\hline
        \citep{askari2018lifted} & $89.8\%$ & $87.5\%$ & $86.5\%$ & $85.3\%$ & $77.0\%$\\
        \hline
        LPOM  & $97.7\%$ & $96.9\%$ & $97.1\%$ & $96.2\%$ & $96.1\%$ \\
        \hline
    \end{tabular}
}{
 \caption{Comparison of accuracies of LPOM and~\citep{askari2018lifted} on the MNIST dataset using different networks.}
 \label{tab:comparison}
}
}
\scalebox{0.93}[0.93]{
\capbtabbox{
 \begin{tabular}{c||c}
 \hline
 SGD & $95.0\%$ \\
 \hline
 \citep{taylor2016training} & $96.5\%$ \\
 \hline
 LPOM & $98.3\%$ \\
 \hline
 \end{tabular}
}{
 \caption{Comparison with SGD and ~\citep{taylor2016training} on the SVHD dataset.}
 \label{tab:comparison_SVHN}
}
}
\end{floatrow}
\end{table*}

In this section we evaluate LPOM by comparing with \mbox{SGD} and two non-gradient based methods~\citep{askari2018lifted,taylor2016training}. 
The other non-gradient based methods do not train fully connected feed-forward  neural networks for classification tasks (e.g., using skip connections~\citep{zhang2017convergent}, training autoencoders~\citep{carreira2014distributed}, and learning for hashing~\citep{zhang2016efficient}). So we cannot include them for comparison. For simplicity, we utilize the least-square loss function and the ReLU activation
function\footnote{The other reason of using ReLU is that it can
produce higher accuracies, although LPOM can compute with other
activation functions without numerical difficulty.} unless
specified otherwise. Unlike \citep{askari2018lifted}, we do not
use any regularization on the weights $\{W^i\}_{i=1}^{n-1}$. We
run LPOM and SGD with the same inputs and random
initializations~\citep{glorot2010understanding}. We implement LPOM
with MATLAB without optimizing the code. We use the SGD based
solver in Caffe~\citep{jia2014caffe}. For the Caffe solver, we
modify the demo code and carefully tune the parameters to achieve
the best performances. For~\citep{askari2018lifted} and~\citep{taylor2016training}, we quote their
results from the papers.



\subsection{Comparison with SGD}
We conduct experiments on two datasets, i.e.,
MNIST~\footnote{\url{http://yann.lecun.com/exdb/mnist/}} and
CIFAR-10~\footnote{\url{https://www.cs.toronto.edu/~kriz/cifar.html}}.
For the MNIST dataset, we use $28\!\times\!28=784$ raw pixels as
the inputs. It includes 60,000 training images and 10,000 test
images. We do not use pre-processing or data augmentation. For
LPOM and SGD, in each epoch the entire training samples are passed
through once. The performance depends the choice of network
architecture. Following~\citep{zeng2018global}, we implement a
784-2048-2048-2048-10 feed-forward neural network. For LPOM, we
simply set $\mu_i\!=\!20$ in~\eqref{problem_LPOM}. We run LPOM and
SGD for 100 epochs with a fixed batch size 100. The training and
test accuracies are shown in Fig.~\ref{fig:comparison} (a) and
(b). We can see that the training accuracies of the two methods
are both approximately equal to $100\%$. However, the test
accuracy of LPOM is slightly better than that of SGD ($98.2\%$ vs.
$98.0\%$).

For the CIFAR-10 dataset, as in~\citep{zeng2018global} we
implement a 3072-4000-1000-4000-10 feed-forward neural network. We
normalize color images by subtracting the training dataset's means
of the red, green, and blue channels, respectively. We do not use
pre-processing or data augmentation. For LPOM, we set
$\mu_i\!=\!100$ in~\eqref{problem_LPOM}. We run LPOM and SGD for
100 epochs with a fixed batch size 100. The training and test
accuracies are shown in Fig.~\ref{fig:comparison} (c) and (d). We
can see that the training accuracies of SGD and LPOM are
approximately equal to $100\%$. However, the test accuracy of LPOM
is better than that of SGD ($52.5\%$ vs. $47.5\%$).

\subsection{Comparison with Other Non-gradient Based Methods}
\label{sec:compare_askari}

We compare against~\citep{askari2018lifted} with identical
architectures on the MNIST dataset. \citet{askari2018lifted} only
use the ReLU activation function in real computation. As
in~\citep{askari2018lifted}, we run LPOM for 17 epochs with a
fixed batch size 100. For LPOM, we set $\mu_i\!=\!20$ for all the
networks. We 
do not use pre-processing or data augmentation. The
test accuracies of the two methods are shown in
Table~\ref{tab:comparison}. We can see that LPOM with the ReLU
activation function performs better than~\citep{askari2018lifted}
with significant gaps. This complies with our analysis in
Subsection ``Advantages of LPOM''.

Following the settings of dataset and network architecture in~\citep{taylor2016training},
we test LPOM on the Street View House Numbers (SVHN) dataset~\citep{netzer2011reading}. For LPOM, we set $\mu_i\!=\!20$
in~\eqref{problem_LPOM}. The test accuracies of SGD, \citep{taylor2016training}, and LPOM are shown in Table~\ref{tab:comparison_SVHN}.
We can see that LPOM outperforms SGD and \citep{taylor2016training}.
This further verifies the advantage of LPOM.


\section{Conclusions}
In this work we have proposed LPOM to train fully connected
feed-forward neural networks. Using the proximal operator, LPOM
transforms the neural network into a new block multi-convex model.
The transformation works for general non-decreasing
Lipschitz continuous activation functions. We propose a novel
block coordinate descent algorithm with convergence guarantee for
each subproblem. LPOM can be solved in parallel and without more
auxiliary variables than the layer-wise activations. Our
experimental results show that LPOM works better than SGD,
\citep{askari2018lifted}, and~\citep{taylor2016training} on fully
connected neural networks. Future work includes extending LPOM to
train convolutional and recurrent neural networks and applying
LPOM to network compression.


\clearpage
\bibliographystyle{aaai}   
\bibliography{LPOM}

\begin{thebibliography}{}

\bibitem[\protect\citeauthoryear{Askari \bgroup et al\mbox.\egroup
  }{2018}]{askari2018lifted}
Askari, A.; Negiar, G.; Sambharya, R.; and Ghaoui, L.~E.
\newblock 2018.
\newblock Lifted neural networks.
\newblock {\em arXiv preprint arXiv:1805.01532}.

\bibitem[\protect\citeauthoryear{Beck and Teboulle}{2009}]{Beck2009FISTA}
Beck, A., and Teboulle, M.
\newblock 2009.
\newblock A fast iterative shrinkage-thresholding algorithm for linear inverse
  problems.
\newblock {\em {SIAM} Journal on Imaging Sciences}  183--202.

\bibitem[\protect\citeauthoryear{Carreira-Perpinan and
  Wang}{2014}]{carreira2014distributed}
Carreira-Perpinan, M., and Wang, W.
\newblock 2014.
\newblock Distributed optimization of deeply nested systems.
\newblock In {\em International Conference on Artificial Intelligence and
  Statistics},  10--19.

\bibitem[\protect\citeauthoryear{Clevert, Unterthiner, and
  Hochreiter}{2015}]{clevert2015fast}
Clevert, D.-A.; Unterthiner, T.; and Hochreiter, S.
\newblock 2015.
\newblock Fast and accurate deep network learning by exponential linear units
  (elus).
\newblock {\em arXiv preprint arXiv:1511.07289}.

\bibitem[\protect\citeauthoryear{Collobert \bgroup et al\mbox.\egroup
  }{2011}]{collobert2011natural}
Collobert, R.; Weston, J.; Bottou, L.; Karlen, M.; Kavukcuoglu, K.; and Kuksa,
  P.
\newblock 2011.
\newblock Natural language processing (almost) from scratch.
\newblock {\em Journal of Machine Learning Research} 12(Aug):2493--2537.

\bibitem[\protect\citeauthoryear{Dauphin, de Vries, and
  Bengio}{2015}]{dauphin2015equilibrated}
Dauphin, Y.; de~Vries, H.; and Bengio, Y.
\newblock 2015.
\newblock Equilibrated adaptive learning rates for non-convex optimization.
\newblock In {\em Advances in Neural Information Processing Systems},
  1504--1512.

\bibitem[\protect\citeauthoryear{Duchi, Hazan, and
  Singer}{2011}]{duchi2011adaptive}
Duchi, J.; Hazan, E.; and Singer, Y.
\newblock 2011.
\newblock Adaptive subgradient methods for online learning and stochastic
  optimization.
\newblock {\em Journal of Machine Learning Research} 12:2121--2159.

\bibitem[\protect\citeauthoryear{Ge \bgroup et al\mbox.\egroup
  }{2015}]{ge2015escaping}
Ge, R.; Huang, F.; Jin, C.; and Yuan, Y.
\newblock 2015.
\newblock Escaping from saddle points{-}online stochastic gradient for tensor
  decomposition.
\newblock In {\em Conference on Learning Theory},  797--842.

\bibitem[\protect\citeauthoryear{Glorot and
  Bengio}{2010}]{glorot2010understanding}
Glorot, X., and Bengio, Y.
\newblock 2010.
\newblock Understanding the difficulty of training deep feedforward neural
  networks.
\newblock In {\em Proceedings of the Thirteenth International Conference on
  Artificial Intelligence and Statistics},  249--256.

\bibitem[\protect\citeauthoryear{Golub and Van~Loan}{2012}]{golub2012matrix}
Golub, G.~H., and Van~Loan, C.~F.
\newblock 2012.
\newblock {\em Matrix Computations}, volume~3.
\newblock The Johns Hopkins University Press.

\bibitem[\protect\citeauthoryear{He \bgroup et al\mbox.\egroup
  }{2016}]{he2016deep}
He, K.; Zhang, X.; Ren, S.; and Sun, J.
\newblock 2016.
\newblock Deep residual learning for image recognition.
\newblock In {\em Proceedings of the IEEE Conference on Computer Vision and
  Pattern Recognition},  770--778.

\bibitem[\protect\citeauthoryear{Hinton \bgroup et al\mbox.\egroup
  }{2012}]{hinton2012deep}
Hinton, G.; Deng, L.; Yu, D.; Dahl, G.~E.; Mohamed, A.-R.; Jaitly, N.; Senior,
  A.; Vanhoucke, V.; Nguyen, P.; Sainath, T.~N.; et~al.
\newblock 2012.
\newblock Deep neural networks for acoustic modeling in speech recognition: The
  shared views of four research groups.
\newblock {\em IEEE Signal Processing Magazine} 29(6):82--97.

\bibitem[\protect\citeauthoryear{Hubara \bgroup et al\mbox.\egroup
  }{2016}]{hubara2016binarized}
Hubara, I.; Courbariaux, M.; Soudry, D.; El-Yaniv, R.; and Bengio, Y.
\newblock 2016.
\newblock Binarized neural networks.
\newblock In {\em Advances in Neural Information Processing Systems},
  4107--4115.

\bibitem[\protect\citeauthoryear{Jia \bgroup et al\mbox.\egroup
  }{2014}]{jia2014caffe}
Jia, Y.; Shelhamer, E.; Donahue, J.; Karayev, S.; Long, J.; Girshick, R.;
  Guadarrama, S.; and Darrell, T.
\newblock 2014.
\newblock Caffe: Convolutional architecture for fast feature embedding.
\newblock In {\em Proceedings of the 22nd ACM International Conference on
  Multimedia},  675--678.
\newblock ACM.

\bibitem[\protect\citeauthoryear{Kingma and Ba}{2014}]{kingma2014adam}
Kingma, D.~P., and Ba, J.
\newblock 2014.
\newblock Adam: A method for stochastic optimization.
\newblock {\em arXiv preprint arXiv:1412.6980}.

\bibitem[\protect\citeauthoryear{Kreyszig}{1978}]{kreyszig1978introductory}
Kreyszig, E.
\newblock 1978.
\newblock {\em Introductory Functional Analysis with Applications}, volume~1.
\newblock Wiley New York.

\bibitem[\protect\citeauthoryear{Krizhevsky, Sutskever, and
  Hinton}{2012}]{krizhevsky2012imagenet}
Krizhevsky, A.; Sutskever, I.; and Hinton, G.~E.
\newblock 2012.
\newblock Imagenet classification with deep convolutional neural networks.
\newblock In {\em Advances in Neural Information Processing Systems},
  1097--1105.

\bibitem[\protect\citeauthoryear{Lin, Liu, and Su}{2011}]{lin2011linearized}
Lin, Z.; Liu, R.; and Su, Z.
\newblock 2011.
\newblock Linearized alternating direction method with adaptive penalty for
  low-rank representation.
\newblock In {\em Advances in Neural Information Processing Systems},
  612--620.

\bibitem[\protect\citeauthoryear{Nesterov}{2004}]{Nesterov2000Introductory}
Nesterov, Y., ed.
\newblock 2004.
\newblock {\em Introductory Lectures on Convex Optimization: A Basic Course}.
\newblock Springer.

\bibitem[\protect\citeauthoryear{Netzer \bgroup et al\mbox.\egroup
  }{2011}]{netzer2011reading}
Netzer, Y.; Wang, T.; Coates, A.; Bissacco, A.; Wu, B.; and Ng, A.~Y.
\newblock 2011.
\newblock Reading digits in natural images with unsupervised feature learning.
\newblock In {\em NIPS workshop on Deep Learning and Unsupervised Feature
  Learning}, volume 2011, ~5.

\bibitem[\protect\citeauthoryear{Parikh, Boyd, and
  others}{2014}]{parikh2014proximal}
Parikh, N.; Boyd, S.; et~al.
\newblock 2014.
\newblock Proximal algorithms.
\newblock {\em Foundations and Trends{\textregistered} in Optimization}
  1(3):127--239.

\bibitem[\protect\citeauthoryear{Rumelhart, Hinton, and
  Williams}{1986}]{rumelhart1986learning}
Rumelhart, D.~E.; Hinton, G.~E.; and Williams, R.~J.
\newblock 1986.
\newblock Learning representations by back-propagating errors.
\newblock {\em Nature} 323(6088):533.

\bibitem[\protect\citeauthoryear{Silver \bgroup et al\mbox.\egroup
  }{2016}]{silver2016mastering}
Silver, D.; Huang, A.; Maddison, C.~J.; Guez, A.; Sifre, L.; Van Den~Driessche,
  G.; Schrittwieser, J.; Antonoglou, I.; Panneershelvam, V.; Lanctot, M.;
  et~al.
\newblock 2016.
\newblock Mastering the game of {G}o with deep neural networks and tree search.
\newblock {\em Nature} 529(7587):484.

\bibitem[\protect\citeauthoryear{Sutskever \bgroup et al\mbox.\egroup
  }{2013}]{sutskever2013importance}
Sutskever, I.; Martens, J.; Dahl, G.; and Hinton, G.
\newblock 2013.
\newblock On the importance of initialization and momentum in deep learning.
\newblock In {\em International Conference on Machine Learning},  1139--1147.

\bibitem[\protect\citeauthoryear{Taylor \bgroup et al\mbox.\egroup
  }{2016}]{taylor2016training}
Taylor, G.; Burmeister, R.; Xu, Z.; Singh, B.; Patel, A.; and Goldstein, T.
\newblock 2016.
\newblock Training neural networks without gradients: A scalable {ADMM}
  approach.
\newblock In {\em International Conference on Machine Learning},  2722--2731.

\bibitem[\protect\citeauthoryear{Zeng \bgroup et al\mbox.\egroup
  }{2018}]{zeng2018global}
Zeng, J.; Ouyang, S.; Lau, T. T.-K.; Lin, S.; and Yao, Y.
\newblock 2018.
\newblock Global convergence in deep learning with variable splitting via the
  {K}urdyka-{L}ojasiewicz property.
\newblock {\em arXiv preprint arXiv:1803.00225}.

\bibitem[\protect\citeauthoryear{Zhang and Brand}{2017}]{zhang2017convergent}
Zhang, Z., and Brand, M.
\newblock 2017.
\newblock Convergent block coordinate descent for training {T}ikhonov
  regularized deep neural networks.
\newblock In {\em Advances in Neural Information Processing Systems},
  1721--1730.

\bibitem[\protect\citeauthoryear{Zhang, Chen, and
  Saligrama}{2016}]{zhang2016efficient}
Zhang, Z.; Chen, Y.; and Saligrama, V.
\newblock 2016.
\newblock Efficient training of very deep neural networks for supervised
  hashing.
\newblock In {\em Proceedings of the IEEE Conference on Computer Vision and
  Pattern Recognition},  1487--1495.

\end{thebibliography}

\newpage
\section*{Supplementary Material of \\Lifted Proximal Operator Machine}

\subsection{Optimality Conditions of~\citep{zeng2018global}}

The optimality conditions of~\citep{zeng2018global} are (obtained
by differentiating the objective function w.r.t. $X^n$,
$\{X^i\}_{i=2}^{n-1}$, $\{W^i\}_{i=1}^{n-1}$, and
$\{U^i\}_{i=2}^n$, respectively):
\begin{equation}
\frac{\partial {\ell}(X^n,L)}{\partial
X^n}\!+\!\mu(X^n\!-\!\phi(U^n))\!=\!\mathbf{0},
\end{equation}
\begin{equation}
(W^i)^T(W^iX^i\!-\!U^{i+1})\!+\!(X^i\!-\!\phi(U^i))\!=\!\mathbf{0},\,i=2,\cdots,n-1,
\end{equation}
\begin{equation}
(W^iX^i\!-\!U^{i+1})(X^i)^T\!=\!\mathbf{0},\,i=1,\cdots,n-1,
\end{equation}
\begin{equation}
(U^{i}\!-\!W^{i-1}X^{i-1})\!+\!(\phi(U^i)\!-\!X^i)\circ\phi'(U^i)\!=\!\mathbf{0},\,i=2,\cdots,n
\end{equation}
where $\circ$ denotes the element-wise multiplication.

\subsection{Proof of Theorem 2}
If $f(x)$ is contractive: $\|f(x)\!-\!f(y)\|\!\leq\! \rho
\|x-y\|$, for all $x$, $y$, where $0\!\leq \!\rho \!< \!1$. Then
the iteration $x_{k+1}\!=\!f(x_{k})$ is convergent and the
convergence rate is linear~\citep{kreyszig1978introductory}. If
$f(x)$ is continuously differentiable, then $\|\nabla f(x)\|
\!\leq \!\rho$ ensures that $f(x)$ is contractive.

Now we need to estimate the Lipschitz coefficient $\rho$ for the
mapping
$X^{i,{t+1}}\!=\!f(X^{i,t})\!=\!\phi\left(W^{i\!-\!1}X^{i\!-\!1}\!-\!\frac{\mu_{i\!+\!1}}{\mu_i}
{(W^{i})}^T(\phi(W^iX^i)\!-\!X^{i\!+\!1})\right)$. Its Jacobian
matrix is:
\begin{equation}
\begin{split}
&J_{kl,pq}=\frac{\partial [f(X^{i,t})]_{kl}}{\partial X^{i,t}_{pq}}\\
&=\!\frac{\partial \phi\left([W^{i\!-\!1}X^{i\!-\!1}]_{kl}\!-\!\frac{\mu_{i\!+\!1}}{\mu_i} [{(W^{i})}^T(\phi(W^iX^{i,t})\!-\!X^{i\!+\!1})]_{kl}\right)  }{\partial X^{i,t}_{pq}}\\
&=\!-\frac{\mu_{i\!+\!1}}{\mu_i}\phi'(c^{i,t}_{kl})\frac{\partial [{(W^{i})}^T(\phi(W^iX^{i,t})\!-\!X^{i\!+\!1})]_{kl}}{\partial X^{i,t}_{pq}}\\
&=\!-\frac{\mu_{i\!+\!1}}{\mu_i}\phi'(c^{i,t}_{kl})\frac{\partial \sum_r W^i_{rk} [\phi\left((W^iX^{i,t})_{rl}\right)\!-\!X^{i\!+\!1}_{rl}]}{\partial X^{i,t}_{pq}}\\
&=\!-\frac{\mu_{i\!+\!1}}{\mu_i}\phi'(c^{i,t}_{kl})\sum_r W^i_{rk} \phi'((W^iX^{i,t})_{rl})\frac{\partial (W^iX^{i,t})_{rl} }{\partial X^{i,t}_{pq}}\\
&=\!-\frac{\mu_{i\!+\!1}}{\mu_i}\phi'(c^{i,t}_{kl})\sum_r W^i_{rk} \phi'((W^iX^{i,t})_{rl})\frac{\partial \sum_s W^i_{rs}X^{i,t}_{sl} }{\partial X^{i,t}_{pq}}\\
&=\!-\frac{\mu_{i\!+\!1}}{\mu_i}\phi'(c^{i,t}_{kl})\sum_r W^i_{rk} \phi'((W^iX^{i,t})_{rl})\sum_s W^i_{rs}\delta_{sp}\delta_{lq}\\
&=\!-\frac{\mu_{i\!+\!1}}{\mu_i}\phi'(c^{i,t}_{kl})\sum_r W^i_{rk}
\phi'((W^iX^{i,t})_{rl}) W^i_{rp}\delta_{lq},
\end{split}
\end{equation}
where
$c^{i,t}_{kl}=[W^{i\!-\!1}X^{i\!-\!1}]_{kl}\!-\!\frac{\mu_{i\!+\!1}}{\mu_i}
[{(W^{i})}^T(\phi(W^iX^{i,t})\!-\!X^{i\!+\!1})]_{kl}$,
$\delta_{sp}$ is the Kronecker delta function, it is 1 if $s$ and
$p$ are equal, and 0 otherwise. Its $l_1$ norm is upper bounded
by:
\begin{equation}
\begin{split}
&\|J\|_1\!=\!\max_{pq}\sum_{kl} |J_{kl,pq}|\\
&=\!\frac{\mu_{i\!+\!1}}{\mu_i}\max_{pq}\sum_{kl}\left|\phi'(c^{i,t}_{kl})\sum_r W^i_{rk} \phi'((W^iX^{i,t})_{rl}) W^i_{rp}\delta_{lq}\right|\\
&\leq \! \frac{\mu_{i\!+\!1}}{\mu_i} \gamma^2 \max_p \sum_k\sum_r|W^i_{rk}||W^i_{rp}|\\
&\leq \!\frac{\mu_{i\!+\!1}}{\mu_i} \gamma^2 \max_p \sum_k \left(|(W^i)^T||W^i|\right)_{kp}\\
&= \!\frac{\mu_{i\!+\!1}}{\mu_i} \gamma^2 \left\| |(W^i)^T| |W^i|
\right\|_1.
\end{split}
\end{equation}
Its $l_\infty$ norm is upper bounded by
\begin{equation}
\begin{split}
&\|J\|_\infty\!=\!\max_{kl}\sum_{pq} |J_{kl,pq}|\\
&=\!\frac{\mu_{i\!+\!1}}{\mu_i}\max_{kl}\sum_{pq}\left|\phi'(c^{i,t}_{kl})\sum_r W^i_{rk} \phi'((W^iX^{i,t})_{rl}) W^i_{rp}\delta_{lq}\right|\\
&\leq\! \frac{\mu_{i\!+\!1}}{\mu_i} \gamma^2 \max_k \sum_p\sum_r|W^i_{rk}||W^i_{rp}|\\
&\leq \! \frac{\mu_{i\!+\!1}}{\mu_i} \gamma^2 \max_k \sum_p \left(|(W^i)^T||W^i|\right)_{kp}\\
&=\! \frac{\mu_{i\!+\!1}}{\mu_i} \gamma^2 \left\| |(W^i)^T| |W^i|
\right\|_\infty.
\end{split}
\end{equation}
Therefore, by using $\|A\|_2\leq
\sqrt{\|A\|_1\|A\|_\infty}$~\citep{golub2012matrix}, the $l_2$
norm of its Jacobian matrix is upper bounded by
\begin{equation}
\|J\|_2\leq\frac{\mu_{i+1}}{\mu_i}\gamma^2\sqrt{\left\|
|(W^i)^T||W^i|\right\|_1\left\| |(W^i)^T| |W^i|\right\|_\infty},
\end{equation}
which is the Lipschitz coefficient $\rho$.

\subsection{Proof of Theorem 3}

The proof of the first part is the same as that of Theorem 2. So
we only detail how to estimate the Lipschitz coefficient $\tau$
for the mapping
$X^{n,t+1}\!=\!f(X^{n,t})\!=\!\phi\left(W^{n-1}X^{n-1}\!-\!\frac{1}{\mu_n}\frac{\partial
{\ell}(X^{n,t},L)}{\partial X^{n,t}}\right)$. Its Jacobian matrix
is:
\begin{equation}
\begin{split}
J_{kl,pq}&=\frac{\partial [f(X^{n,t})]_{kl}}{\partial X^{n,t}_{pq}}\\
&=\frac{\partial  \phi\left( (W^{n-1}X^{n-1})_{kl}\!-\!\frac{1}{\mu_n}\frac{\partial {\ell}(X^{n,t},L)}{\partial X^{n,t}_{kl}}\right) }{\partial X^{n,t}_{pq}}\\
&=-\frac{1}{\mu_n}\phi'(d^{n,t}_{kl})\frac{\partial \frac{\partial {\ell}(X^{n,t},L)}{\partial X^{n,t}_{kl}} }{\partial X^{n,t}_{pq}}\\
&=-\frac{1}{\mu_n}\phi'(d^{n,t}_{kl}) \frac{\partial^2
{\ell}(X^{n,t},L)}{\partial X^{n,t}_{kl} \partial X^{n,t}_{pq}},
\end{split}
\end{equation}
where $d^{n,t}_{kl}=(W^{n-1}X^{n-1})_{kl}\!-\!\frac{1}{\mu_n}
\left(\frac{\partial {\ell}(X^{n,t},L)}{\partial
X^{n,t}}\right)_{kl}$. Its $l_1$ norm is upper bounded by:
\begin{equation}
\begin{split}
\|J\|_1&=\max_{pq}\sum_{kl} |J_{kl,pq}|\\
&=\frac{1}{\mu_n}\max_{pq}\sum_{kl}\left|\phi'(d^{n,t}_{kl})\frac{\partial^2 {\ell}(X^{n,t},L)}{\partial X^{n,t}_{kl} \partial X^{n,t}_{pq}}\right|\\
&\leq \frac{\gamma}{\mu_n}\max_{pq}\sum_{kl} \left|\frac{\partial^2 {\ell}(X^{n,t},L)}{\partial X^{n,t}_{kl} \partial X^{n,t}_{pq}}\right| \\
&=\frac{\gamma}{\mu_n} \left\|\left|\frac{\partial^2 {\ell}(X^{n,t},L)}{\partial X^{n,t}_{kl} \partial X^{n,t}_{pq}}\right| \right\|_1 \\
&\leq \frac{\gamma \eta}{\mu_n}.
\end{split}
\end{equation}
Its $l_\infty$ norm is upper bounded by:
\begin{equation}
\begin{split}
\|J\|_\infty&=\max_{kl}\sum_{pq} |J_{kl,pq}|\\
&=\frac{1}{\mu_n}\max_{kl}\sum_{pq}\left|\phi'(d^{n,t}_{kl})\frac{\partial^2 {\ell}(X^{n,t},L)}{\partial X^{n,t}_{kl} \partial X^{n,t}_{pq}}\right|\\
&\leq \frac{\gamma}{\mu_n}\max_{kl}\sum_{pq} \left|\frac{\partial^2 {\ell}(X^{n,t},L)}{\partial X^{n,t}_{kl} \partial X^{n,t}_{pq}}\right| \\
&=\frac{\gamma}{\mu_n} \left\|\left|\frac{\partial^2 {\ell}(X^{n,t},L)}{\partial X^{n,t}_{kl} \partial X^{n,t}_{pq}}\right| \right\|_1 \\
&\leq \frac{\gamma \eta}{\mu_n}.
\end{split}
\end{equation}
Therefore, the $l_2$ norm of $J$ is upper bounded by
\begin{equation}
\|J\|_2\leq\sqrt{\|J\|_1\|J\|_\infty}\!\leq\!\frac{\gamma\eta}{\mu_n}\!=\!\tau.
\end{equation}

\subsection{Proof of Theorem 4}
The $L_\varphi$-smoothness of $\varphi$: $$\|\nabla
\varphi(x)\!-\!\nabla \varphi(y)\| \!\leq \!L_\varphi\|x\!-\!y\|,
\forall x,y$$ enables the following inequality
\citep{Nesterov2000Introductory}:
\begin{equation}\label{eq:L_smooth}
\varphi(z)\leq \varphi(y)\!+\!\langle \nabla
\varphi(y),z\!-\!y\rangle\!+\!\frac{L_\varphi}{2}\|z\!-\!y\|^2,
\forall x,y.
\end{equation}
By putting $z=Ax$ and $y=Ay_k$, where $y_k$ is yet to be chosen,
we have
\begin{equation}\label{eq:inequality_by_L_smooth}
\varphi(Ax)\leq\varphi(Ay_k)\!+\!\langle \nabla
\varphi(Ay_k),A(x\!-\!y_k)\rangle\!+\!\frac{L_\varphi}{2}\|A(x\!-\!y_k)\|^2.
\end{equation}
As assumed,
\begin{equation}\label{eq:update_x'}
x_{k+1}\!=\!\argmin_x \langle \nabla
\varphi(Ay_k),A(x\!-\!y_k)\rangle\!+\!\frac{L_\varphi}{2}\|A(x\!-\!y_k)\|^2\!+\!h(x)
\end{equation}
is easy to solve. This gives
\begin{equation}\label{eq:partial_h}
-L_\varphi A^TA(x_{k+1}\!-\!y_k)\in\!A^T\nabla
\varphi(Ay_k)\!+\!\partial h(x_{k+1}).
\end{equation}

Then by \eqref{eq:inequality_by_L_smooth} and the convexity of
$h$, we have
\begin{equation}
\begin{split}
&F(x_{k+1})\!=\!\varphi(Ax_{k+1})\!+\!h(x_{k+1})\\
&\leq \varphi(Ay_k)\!+\!\langle \nabla \varphi(Ay_k),A(x_{k+1}\!-\!y_k)\rangle\!+\!\frac{L_\varphi}{2}\|A(x_{k+1}\!-\!y_k)\|^2\\
&\quad\!+\!h(u)
\!-\!\langle\xi,u\!-\!x_{k+1}\rangle\\
&\leq \varphi(Au)\!+\!\langle \nabla \varphi(Ay_k),A(u\!-\!y_k)\rangle\!+\!\langle \nabla \varphi(Ay_k),A(x_{k+1}\!-\!y_k)\rangle\\
&\quad\!+\!\frac{L_\varphi}{2}\|A(x_{k+1}\!-\!y_k)\|^2\!+\!h(u)
\!-\!\langle\xi,u\!-\!x_{k+1}\rangle\\
&=F(u)\!-\!\langle A^T \nabla \varphi(Ay_k)\!+\!\xi,u\!-\!x_{k+1}\rangle\!+\!\frac{L_\varphi}{2}\|A(x_{k+1}\!-\!y_k)\|^2\\
&=F(u)\!+\!L_\varphi\langle A^TA(x_{k+1}\!-\!y_k),u\!-\!x_{k+1}\rangle\!+\!\frac{L_\varphi}{2}\|A(x_{k+1}\!-\!y_k)\|^2\\
&=F(u)\!+\!L_\varphi\langle
A(x_{k+1}\!-\!y_k),A(u\!-\!x_{k+1})\rangle\!+\!\frac{L_\varphi}{2}\|A(x_{k+1}\!-\!y_k)\|^2,
\end{split}
\end{equation}
where $\xi$ is any subgradient in $\partial h(x_{k+1})$,  $u$ is
any point, and the third equality used~\eqref{eq:partial_h}. Thus
\begin{equation}\label{eq:inequality_x}
\begin{split}
F(x_{k+1})\!&\leq\! F(u)\!+\!L_g\langle A(x_{k+1}\!-\!y_k),A(u\!-\!x_{k+1})\rangle\\
&~~\!+\!\frac{L_g}{2}\|A(x_{k+1}\!-\!y_k)\|^2,\quad \forall u.
\end{split}
\end{equation}
Let $u\!=\!x_k$ and $u\!=\!x^*$ in~\eqref{eq:inequality_x},
respectively. Then multiplying the first inequality with
$\theta_k$ and the second with $1\!-\!\theta_k$ and adding them
together, we have
\begin{equation}\label{eq:inequaltiy_F}
\begin{split}
&F(x_{k+1})\!\leq\! \theta_k F(x_k)\!+\!(1-\theta_k) F(x^*)\\
&~~\!+\!L_\varphi\langle A(x_{k+1}\!-\!y_k),A[\theta_k(x_k\!-\!x_{k+1})\!+\!(1\!-\!\theta_k)(x^*\!-\!x_{k+1})]\rangle\\
&~~\!+\!\frac{L_\varphi}{2}\|A(x_{k+1}\!-\!y_k)\|^2\\
&=\!\theta_kF(x_k)\!+\!(1\!-\!\theta_k)F(x^*)\\
&~~\!+\!L_\varphi\langle A(x_{k+1}\!-\!y_k),A[\theta_k x_k\!-\!x_{k+1}\!+\!(1\!-\!\theta_k)x^*]\rangle\\
&~~\!+\!\frac{L_\varphi}{2}\|A(x_{k+1}\!-\!y_k)\|^2\\
&=\!\theta_kF(x_k)\!+\!(1\!-\!\theta_k)F(x^*)\\
&~~\!+\!\frac{L_\varphi}{2}\left\{\|A[(x_{k+1}\!-\!y_k)\!+\!(\theta_k x_k\!-\!x_{k+1}\!+\!(1\!-\!\theta_k)x^*)]\|^2\right.\\
&~~-\left.\|A(x_{k+1}\!-\!y_k)\|^2\!-\!\|A[\theta_kx_k-x_{k+1}\!+\!(1\!-\!\theta_k)x^*]\|^2\right\}\\
&~~\!+\!\frac{L_\varphi}{2}\|A(x_{k+1}\!-\!y_k)\|^2\\
&=\!\theta_kF(x_k)\!+\!(1\!-\!\theta_k)F(x^*)\\
&~~\!+\!\frac{L_\varphi}{2}\left\{\|A[\theta_k x_k\!-\!y_{k}\!+\!(1\!-\!\theta_k)x^*]\|^2\right.\\
&~~\left.-\|A[\theta_kx_k-x_{k+1}\!+\!(1\!-\!\theta_k)x^*]\|^2\right\}.
\end{split}
\end{equation}
In order to have a recursion, we need to have:
$$\theta_kx_k\!-\!y_k\!+\!(1\!-\!\theta_k)x^*\!=\!\sqrt{\theta_k}[\theta_{k-1}x_{k-1}\!-\!x_k+(1\!-\!\theta_{k-1})x^*].$$
By comparing the coefficient of $x^*$, we have
\begin{equation}\label{eq:update_theta}
1\!-\!\theta_k\!=\!\sqrt{\theta_k}(1\!-\!\theta_{k-1}).
\end{equation}
Accordingly,
\begin{equation}\label{eq:update_yk}
y_k\!=\!\theta_kx_k\!-\!\sqrt{\theta_k}(\theta_{k-1}x_{k-1}\!-\!x_k).
\end{equation}
With the above choice of $\{\theta_k\}$ and $y_k$,
\eqref{eq:inequaltiy_F} can be rewritten as
\begin{equation}
\begin{split}
&F(x_{k+1})\!-\!F(x^*)\!+\!\frac{L_\varphi}{2}\|z_{k+1}\|^2\\
\leq&\theta_k\left(F(x_k)\!-\!F(x^*)\!+\!\frac{L_\varphi}{2}\|z_k\|^2\right),
\end{split}
\end{equation}
where
$z_k\!=\!A[\theta_{k-1}x_{k-1}\!-\!x_k+(1\!-\!\theta_{k-1})x^*]$.
Then by recursion, we have
\begin{equation}
\begin{split}
&F(x_{k})\!-\!F(x^*)\!+\!\frac{L_\varphi}{2}\|z_{k}\|^2\\
\leq&\left(\prod_{i=1}^{k-1}\theta_i\right)\left(F(x_{1})\!-\!F(x^*)\!+\!\frac{L_\varphi}{2}\|z_{1}\|^2\right).
\end{split}
\end{equation}

It remains to estimate $\prod_{i=1}^{k-1}\theta_i$. We choose
$\theta_0\!=\!0$ and prove
\begin{equation}\label{eq:theta}
1\!-\!\theta_k\!<\!\frac{2}{k+1}
\end{equation}
by induction. \eqref{eq:theta} is true for $k\!=\!0$. Suppose
\eqref{eq:theta} is true for $k\!-\!1$, then by
$1\!-\!\theta_k\!=\!\sqrt{\theta_k}(1\!-\!\theta_{k-1})$, we have
\begin{equation}
1\!-\!\theta_k\!=\!\sqrt{\theta_k}(1\!-\!\theta_{k-1})\!<\!\sqrt{\theta_k}\frac{2}{k}.
\end{equation}
Let $\tilde{\theta}_k\!=\!1\!-\!\theta_k$, then the above becomes
$k^2\tilde{\theta}_k^2\!<\!4(1-\tilde{\theta}_k)$. So
\begin{equation}
\tilde{\theta}_k<\frac{-4+\sqrt{16+16k^2}}{2k^2}\!=\!\frac{2}{1\!+\!\sqrt{1+k^2}}\!<\!\frac{2}{k+1}.
\end{equation}
Thus~\eqref{eq:theta} is proven.

Now we are ready to estimate $\prod_{i=1}^{k-1}\theta_i$. From
$1-\theta_k=\sqrt{\theta_k}(1-\theta_{k-1})$, we have
$$1-\theta_{k-1}\!=\!\sqrt{\prod_{i=1}^{k-1}\theta_i}(1\!-\!\theta_0)\!=\!\sqrt{\prod_{i=1}^{k-1}\theta_i}.$$ So $\prod_{i=1}^{k-1}\theta_i\!=\!(1\!-\!\theta_{k-1})^2\!<\!\frac{4}{k^2}$. Hence
$$
F(x_k)\!-\!F(x^*)\!+\!\frac{L_\varphi}{2}\|z_k\|^2\!\leq\!\frac{4}{k^2}\!\left(F(x_1)\!-\!F(x^*)\!+\!\frac{L_\varphi}{2}\|z_1\|^2\right).
$$

The three equations, \eqref{eq:update_theta},
\eqref{eq:update_yk}, and \eqref{eq:update_x'} constitute the
major steps in Algorithm 2.

\end{document}